\documentclass[11pt]{article}

\usepackage[preprint]{acl}
\usepackage{times}
\usepackage{latexsym}
\usepackage[T1]{fontenc}
\usepackage[utf8]{inputenc}
\usepackage{microtype}
\usepackage{inconsolata}

\usepackage{graphicx}
\usepackage{xcolor}
\usepackage{booktabs}
\usepackage{multirow}
\usepackage{colortbl}
\usepackage{makecell}
\usepackage{tabularx}
\usepackage{array}
\usepackage{amsmath,amssymb,amsfonts,amsthm}
\usepackage{bm,bbm,dsfont,mathrsfs}
\usepackage{algorithm}
\usepackage{algpseudocode}
\usepackage{enumitem}
\usepackage{tikz}
\usetikzlibrary{positioning,arrows.meta,calc,shapes.geometric,fit,backgrounds,decorations.pathmorphing,patterns,shadings}
\usepackage{pifont}
\usepackage{xspace}

\definecolor{tridentblue}{HTML}{1F4E79}
\definecolor{tridentred}{HTML}{B22222}
\definecolor{tridentgreen}{HTML}{1E6B3A}
\definecolor{tridentamber}{HTML}{C77600}
\definecolor{tablerow}{HTML}{EEF3F8}
\definecolor{tablehead}{HTML}{D6E1EC}
\definecolor{oursrow}{HTML}{FFF1D6}
\definecolor{softgrey}{HTML}{F4F4F4}

\newtheorem{theorem}{Theorem}
\newtheorem{lemma}[theorem]{Lemma}

\newtheorem{assumption}[theorem]{Assumption}
\newtheorem{principle}{Design Principle}

\newcommand{\R}{\mathbb{R}}
\newcommand{\E}{\mathbb{E}}

\newcommand{\calS}{\mathcal{S}}\newcommand{\calA}{\mathcal{A}}
\newcommand{\calP}{\mathcal{P}}
\newcommand{\calN}{\mathcal{N}}\newcommand{\calB}{\mathcal{B}}
\newcommand{\calM}{\mathcal{M}}\newcommand{\calL}{\mathcal{L}}
\newcommand{\calO}{\mathcal{O}}
\DeclareMathOperator*{\argmax}{arg\,max}

\DeclareMathOperator{\softmax}{softmax}
\DeclareMathOperator{\diag}{diag}

\newcommand{\norm}[1]{\left\|#1\right\|}

\newcommand{\method}{\textsc{Trident}\xspace}
\newcommand{\sha}{\textsc{Sha}\xspace}
\newcommand{\lcpo}{\textsc{Lcpo}\xspace}
\newcommand{\pirc}{\textsc{Pirc}\xspace}
\newcommand{\stgc}{\textsc{Stgc}\xspace}
\newcommand{\best}[1]{\textbf{\color{tridentblue}#1}}
\newcommand{\secondbest}[1]{\underline{#1}}
\newcommand{\ours}{\rowcolor{oursrow}\,$\bigstar$\,\textbf{\method (Ours)}}
\newcommand{\up}{\textcolor{tridentgreen}{$\uparrow$}}
\newcommand{\down}{\textcolor{tridentred}{$\downarrow$}}

\title{\method: Breaking the Hybrid--Safety--Physics Coupling \\for Provably Safe Multi-Agent Reinforcement Learning}

\author{
  Zijie Meng\thanks{Equal contribution.}\thanks{Corresponding author: \texttt{ymlf@stu.pku.edu.cn}}, \;
  Ziwei Li, \; Zhiyu Li, \; Jiyuan Liu \\
  Peking University
  \AND
  Yufei Liu\footnotemark[1] \\
  Xiamen University
  \And
  Wenhua Nie \\
  National Taiwan University
  \And
  Bingcai Wei \\
  WHU
  \And
  Miao Zhang \\
  THU \;/\; Jimei University
}

\date{}

\begin{document}
\maketitle

\begin{abstract}
Safe coordination in networked cyber-physical systems forces learning algorithms to simultaneously handle \emph{hybrid discrete--continuous actions}, \emph{hard training-time safety constraints}, and \emph{physics-governed dynamics}. We show that these three features form a directed cycle of biases that defeats any naive composition of off-the-shelf modules, and formalize this as a \emph{three-way coupling lemma}. We then introduce \method, the first MARL framework whose three components are co-designed to cancel each leak: a Richardson--Romberg gradient correction reducing Gumbel-Softmax bias from $\mathcal{O}(\tau)$ to $\mathcal{O}(\tau^{2})$, a Lyapunov-constrained sequential trust-region update enforcing per-iterate feasibility, and a physics-informed \emph{residual critic} that decomposes value rather than reward. We prove an $\tilde{\mathcal{O}}(1/\sqrt{K})$ convergence rate to a constrained Nash equilibrium and an $\mathcal{O}(\sqrt{K})$ cumulative-violation bound. On multi-UAV mobile-edge computing, autonomous intersection management, and a hybrid SMAC variant, \method cuts training-time violations by $95.5\%$ over MADDPG and $76.3\%$ over MACPO, while improving reward by $13.5\%$ over the strongest unconstrained baseline.
\end{abstract}

\section{Introduction}
\label{sec:intro}

Consider a fleet of unmanned aerial vehicles (UAVs) deployed over a disaster area to provide mobile edge computing services to ground first responders \citep{wang2021maddpg_uav,zhou2023uav_survey}. Within tens of milliseconds, each UAV must select which of several heterogeneous backbone servers will relay a rescue worker's video stream (a discrete choice over a small set of links), decide how much of the upcoming computation to offload (a continuous fraction in $[0,1]$), and update its trajectory while keeping battery, coverage, and inter-UAV separation strictly within hardware limits. Unlike a recommender or a chess engine that can afford a regret-then-improve curve, every unsafe action committed during training has physical, irreversible consequences---a depleted battery, a mid-air near-miss, a dropped emergency video stream \citep{garcia2015comprehensive,brunke2022safe}. This scene is not exotic; it is the prototypical pattern of safe coordination in \emph{networked cyber-physical systems} (CPS), shared by autonomous-intersection vehicles, robot warehouses, and connected-vehicle platoons \citep{zhou2021smarts,meng2026argus, liu2025synpo, wei2025robust}.

A close look at such systems reveals that three structural features appear together, never in isolation. The first is a hybrid action structure (\textbf{F1}): the decision factorizes as $a=(a^d,a^c)$ in which $a^d$ names a mode (which server, which lane, which target) and $a^c$ parameterizes its execution (offload ratio, throttle, aim point); discretizing $a^c$ destroys resolution while relaxing $a^d$ produces infeasible interpolations between physically incompatible modes \citep{fu2019deep_mapqn,fan2019hybrid}. The second is hard, training-time safety (\textbf{F2}): cost thresholds must be respected not only at convergence but at every iterate the policy is allowed to execute on hardware \citep{achiam2017cpo,chow2018lyapunov,gu2021macpo,gu2024safe_marl_gne}. The third is physics-governed dynamics (\textbf{F3}): substantial portions of the transition kernel and reward follow closed-form physics---Shannon capacity, Friis path loss, Newton's equations---and rediscovering them from scratch wastes orders of magnitude of samples \citep{karniadakis2021physics,banerjee2023physics,cao2024physics, meng2026decouplingsemanticsdistortionsmultiscale, meng2025orpaint, liu2026omnidirector}.

A natural first reaction is to take an off-the-shelf hybrid-action MARL method \citep{fu2019deep_mapqn}, wrap it in a safety procedure such as MACPO \citep{gu2021macpo}, and add a physics-shaped reward term. We tried exactly this composition; the result is unstable, often \emph{worse} than each component in isolation. The root cause, which we make precise in Section~\ref{sec:coupling}, is that the three features form a tight directed cycle of errors rather than a list of independent issues (as illustrated in Figure~\ref{fig:coupling}). Standard Gumbel-Softmax estimators carry an $\mathcal{O}(\tau)$ gradient bias \citep{jang2017gumbel,maddison2017concrete}; substituted into a Lagrangian or trust-region safety update, this bias produces a multiplier that oscillates instead of decreasing the Lyapunov function, so safety guarantees that hold under an exact gradient cease to hold under the biased one (F1$\to$F2). A physics-agnostic safety critic must regress cost-value functions that are highly multi-modal across discrete branches---offloading to fog server 1 versus 2 yields qualitatively different energy curves---and without physics priors it underestimates feasibility margins on rarely visited branches, inducing recovery to over-correct on the wrong branch (F2$\to$F3). Conversely, the standard remedy of folding physics into a single scalar reward-shaping term shifts the soft-Bellman fixed point and destroys the per-branch structure the discrete sub-policy is meant to exploit, so the discrete head learns degenerate, single-mode behaviour (F3$\to$F1). These dependencies form a directed cycle: any module designed in isolation leaks errors into the next, which leaks them back. Treating the three challenges separately is therefore not merely suboptimal---it is provably circular.

We argue that the right level of abstraction is neither ``add safety to MARL'' nor ``add physics to safe RL'', but a joint object: a constrained hybrid-action policy whose gradients are shaped by physics and whose updates are shaped by Lyapunov constraints. The three-way coupling above gives three concrete design principles, each instantiated as one component of \method (the \textbf{T}emperature-corrected, \textbf{R}esidual, \textbf{I}nfinitesimally feasible, \textbf{DE}coupled, seque\textbf{NT}ial framework). Because the components are co-designed, the residual error of one no longer enters the others' guarantees, and a single convergence-and-safety analysis closes the loop. Concretely, our contributions are fourfold:
\begin{itemize}[leftmargin=1.1em,itemsep=2pt,topsep=2pt]
\item \textbf{A coupling lemma} that formalizes why hybrid actions, hard safety, and physics priors cannot be composed naively, and uniquely determines the architecture of any correct fix.
\item \textbf{\method}, the first MARL framework that co-designs hybrid-action, safety, and physics modules so their residual errors no longer feed into one another's guarantees.
\item \textbf{Joint guarantees}: $\tilde{\mathcal{O}}(1/\sqrt{K})$ convergence to a constrained Nash equilibrium, $\mathcal{O}(\sqrt{K})$ cumulative violation, and a physics-driven sample-complexity reduction.
\item \textbf{Strong empirical results} on UAV mobile-edge computing, autonomous intersection management, and a hybrid SMAC variant: $95.5\%$ fewer violations than MADDPG, $76.3\%$ fewer than MACPO, $13.5\%$ higher reward, scaling to 32 agents.
\end{itemize}

\section{Related Work}
\label{sec:related}

\textbf{Hybrid-action MARL.} Deep MAPQN or MAHHQN \citep{fu2019deep_mapqn} pioneered DRL on discrete--continuous spaces, and subsequent work refines parameterized-action factorizations \citep{wang2022hybrid,fan2019hybrid}. None provides convergence rates or safety guarantees, and all rely on standard Gumbel-Softmax estimators \citep{jang2017gumbel,maddison2017concrete} whose $\mathcal{O}(\tau)$ gradient bias is precisely the source of the F1$\to$F2 leakage we identify.

\textbf{Safe MARL.} MACPO \citep{gu2021macpo} extends CPO \citep{achiam2017cpo} with multi-agent trust-region updates and monotonic-improvement guarantees; MAPPO-Lagrangian, built on Safety Gym \citep{ray2019safetygym,yu2022mappo}, only guarantees feasibility \emph{at convergence}; the most recent MADAC \citep{gu2024safe_marl_gne} establishes generalized-Nash convergence but inherits the unbiased-gradient assumption that F1 violates; and shielding methods \citep{elsayed2021shield,alshiekh2018shield} guarantee learning-time safety only with hand-designed shields and cannot accommodate hybrid actions. The Lyapunov-based approach of \citet{chow2018lyapunov} and its extensions \citep{huh2020lyapunov} is the closest precedent for our safety mechanism, but is restricted to single-agent, continuous-action problems and assumes an unbiased policy gradient.

\textbf{Physics-informed and residual RL.} Residual policy learning \citep{silver2018residual,johannink2019residual} composes a model-based prior with a learned correction; physics-regulated DRL \citep{cao2024physics} and physics-informed MBRL \citep{ramesh2023physics} exploit known dynamics. Their reward-shaping variants, however, suffer the F3$\to$F1 leakage we identify, since additive shaping shifts the soft-Bellman fixed point \citep{ng1999policy_invariance}. We adapt residual ideas to a centralized multi-agent critic and quantify, for the first time, the resulting variance reduction in a constrained-MARL setting.


\section{Preliminaries}
\label{sec:prelim}

We model CPS coordination as a \emph{Constrained Multi-Agent MDP} (C-MAMDP) $\calM\!=\!(\mathcal{N},\calS,\{\calA_i\}_i,\calP,r,\{c_k,d_k\}_{k=1}^K,\gamma)$ with $N$ agents, global state $\calS$, hybrid per-agent action space $\calA_i\!=\!\calA_i^d\!\times\!\calA_i^c$ ($\calA_i^d\!=\!\{1,\ldots,M_i\}$ discrete, $\calA_i^c\!\subseteq\!\R^{p_i}$ continuous), kernel $\calP$, shared reward $r$, $K$ bounded costs $c_k\!:\!\calS\!\times\!\calA\!\to\![0,C_{\max}]$ with thresholds $d_k$, and discount $\gamma\!\in\!(0,1)$. Each agent $i$ holds a local policy $\pi_i\!:\!\calO_i\!\to\!\Delta(\calA_i)$ on observation $o_i$; we adopt the standard centralized-training, decentralized-execution (CTDE) paradigm \citep{lowe2017maddpg}, where centralized critics access the full state in training while actors rely solely on local observations at deployment.

Given a joint policy $\bm\pi$, the value and per-constraint cost-value functions are $V^{\bm\pi}\!(s)\!=\!\E_{\bm\pi}\!\big[\sum_t\gamma^t r_t|s_0\!=\!s\big]$ and $V_{c_k}^{\bm\pi}\!(s)\!=\!\E_{\bm\pi}\!\big[\sum_t\gamma^t c_k(s_t,\bm a_t)|s_0\!=\!s\big]$. The objective is a \emph{constrained Nash equilibrium} (CNE): $\max_{\bm\pi}\E_{s_0\sim\rho}V^{\bm\pi}(s_0)$ s.t.\ $\E_{s_0}V_{c_k}^{\bm\pi}(s_0)\!\le\!d_k$ for all $k$, i.e.\ a joint policy from which no agent can unilaterally improve its constrained return---the equilibrium concept used in recent safe-MARL theory \citep{gu2021macpo,gu2024safe_marl_gne}. For hybrid actions, we factorize $\pi_i(a_i|o_i)\!=\!\pi_i^d(a_i^d|o_i)\,\pi_i^c(a_i^c|o_i,a_i^d)$, with $\pi_i^d$ categorical and $\pi_i^c$ a Gaussian conditioned on the discrete choice. This conditional---rather than joint or product---factorization is essential because the continuous parameters change meaning across discrete modes: the same scalar ``power'' carries different physical units on different communication links, so a single shared continuous head would entangle physically incompatible regimes.

\section{The Three-Way Coupling Challenge}
\label{sec:coupling}

This section formalises the intuition of \S\ref{sec:intro}: features (F1)--(F3) induce a directed cycle of bias that closes through the actor, the safety critic, and the reward critic of any naive composition. Quantifying the cycle (Lemma~\ref{lem:coupling}) directly dictates the form of \method.

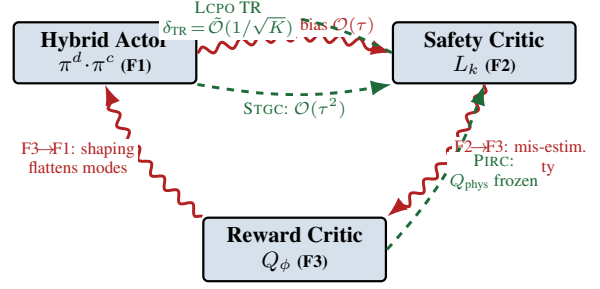
\begin{figure}[t]
\centering
\resizebox{0.99\columnwidth}{!}{%
\begin{tikzpicture}[
  >=Latex,
  mod/.style={draw, line width=0.9pt, rectangle, rounded corners=2pt, minimum height=0.85cm, minimum width=2.6cm, font=\small\bfseries, align=center, fill=tablehead},
  arr/.style={->, very thick, color=tridentred, decorate, decoration={snake, amplitude=.6mm, segment length=2.8mm, post length=1.5mm}},
  fix/.style={->, very thick, color=tridentgreen, dashed},
  lab/.style={font=\scriptsize, align=center, fill=white, inner sep=1pt}
]
\node[mod] (A) at (0,0)     {Hybrid Actor\\\footnotesize $\pi^d\!\cdot\!\pi^c$ {\scriptsize (F1)}};
\node[mod] (B) at (5.4,0)   {Safety Critic\\\footnotesize $L_k$ {\scriptsize (F2)}};
\node[mod] (C) at (2.7,-2.8){Reward Critic\\\footnotesize $Q_\phi$ {\scriptsize (F3)}};
\draw[arr] (A.east)         to[bend left=15]  node[lab,above]      {F1$\!\to\!$F2: GS bias $\mathcal{O}(\tau)$} (B.west);
\draw[arr] (B.south)        to[bend left=15]  node[lab,right=2pt]  {F2$\!\to\!$F3: mis-estim.\\feasibility} (C.north east);
\draw[arr] (C.north west)   to[bend left=15]  node[lab,left=2pt]   {F3$\!\to\!$F1: shaping\\flattens modes} (A.south);
\draw[fix] (A.south east)   to[bend right=10] node[lab,below]      {\stgc: $\mathcal{O}(\tau^2)$} (B.south west);
\draw[fix] (B.west)         to[bend right=18] node[lab,left]       {\lcpo TR\\$\delta_\text{TR}\!=\!\tilde{\mathcal{O}}(1/\sqrt K)$} (A.north east);
\draw[fix] (C.east)         to[bend right=12] node[lab,right]      {\pirc:\\$Q_\text{phys}$ frozen} (B.south);
\end{tikzpicture}}
\caption{\textbf{Three-way coupling.} \textcolor{tridentred}{Red wavy arrows}: the bias-leakage cycle of any naive composition; \textcolor{tridentgreen}{green dashed arrows}: the three co-designed mechanisms in \method that cancel each leak (Lemma~\ref{lem:coupling}).}
\label{fig:coupling}
\end{figure}

Let $\beta_\text{GS}\!:=\!\norm{\E[\hat g^d]-g^d}$ be the discrete-branch gradient bias, $\epsilon_Q$ the reward-critic MSE, and $\eta_s$ the safety-step magnitude; Figure~\ref{fig:coupling} sketches their dependencies.

\begin{lemma}[Bias Propagation in Naive Composition]
\label{lem:coupling}
For a baseline using a Gumbel-Softmax estimator with bias $\beta_\text{GS}$, a Lagrangian or trust-region safety step of magnitude $\eta_s$, and a critic with MSE $\epsilon_Q$, the per-iteration constraint-violation increment satisfies
\begin{equation}
\Delta_k^t\!\le\!\underbrace{\eta_s\beta_\text{GS}\norm{\nabla_a A_{c_k}}_{\!\infty}}_{\text{F1}\to\text{F2}\text{ leak}}\!+\!\underbrace{\eta_s\epsilon_Q L_\pi}_{\text{F2}\to\text{F3}\text{ leak}}\!+\!\underbrace{C_\text{flat}\norm{Q_\text{phys}\!-\!Q^*}_{\!\infty}}_{\text{F3}\to\text{F1}\text{ leak}},
\label{eq:coupling}
\end{equation}
where $L_\pi$ is the policy Lipschitz constant and $C_\text{flat}$ measures the entropy-flattening of misspecified shaping. (Proof: Appendix~\ref{app:coupling}.)
\end{lemma}

Under any drop-in combination---$\beta_\text{GS}\!=\!\Theta(1)$ for Straight-Through and $\Theta(\tau_\text{min})$ for floored Gumbel-Softmax, $\epsilon_Q\!=\!\Theta(1)$ without a physics prior, $C_\text{flat}$ uncontrolled under additive shaping---each summand of \eqref{eq:coupling} is bounded away from zero and cumulative violation grows as $\Theta(K)$. Crucially, the three leaks scale with three \emph{independent} biases, so closing the cycle requires three co-designed---not interchangeable---interventions, one per term: a discrete estimator with $\beta_\text{GS}\!=\!o(\tau)$, a safety step preserving feasibility \emph{per iterate} rather than only asymptotically, and a physics prior entering \emph{multiplicatively} as a frozen value rather than additively as shaping. These three conditions translate one-to-one into the three modules of \method:

\begin{principle}[Bias-attenuated discrete gradients]\label{prn:p1}
$\beta_\text{GS}\!=\!o(\tau)$; instantiated by \stgc (\S\ref{sec:sha}), attaining $\mathcal{O}(\tau^2)$.
\end{principle}
\begin{principle}[Per-iterate feasibility]\label{prn:p2}
Lyapunov trust region with explicit recovery, in place of asymptotic duality; instantiated by \lcpo (\S\ref{sec:lcpo}), yielding $\mathcal{O}(\sqrt{K})$ cumulative violation.
\end{principle}
\begin{principle}[Multiplicative physics prior]\label{prn:p3}
$Q_\phi\!=\!Q_\text{phys}\!+\!Q_{\phi_\text{res}}$ with $Q_\text{phys}$ frozen; instantiated by \pirc (\S\ref{sec:pirc}), attaining $C_\text{flat}\!=\!0$ and shrinking $\epsilon_Q$.
\end{principle}



\begin{figure*}[t]
\centering
\begin{center}
\includegraphics[width=\linewidth]{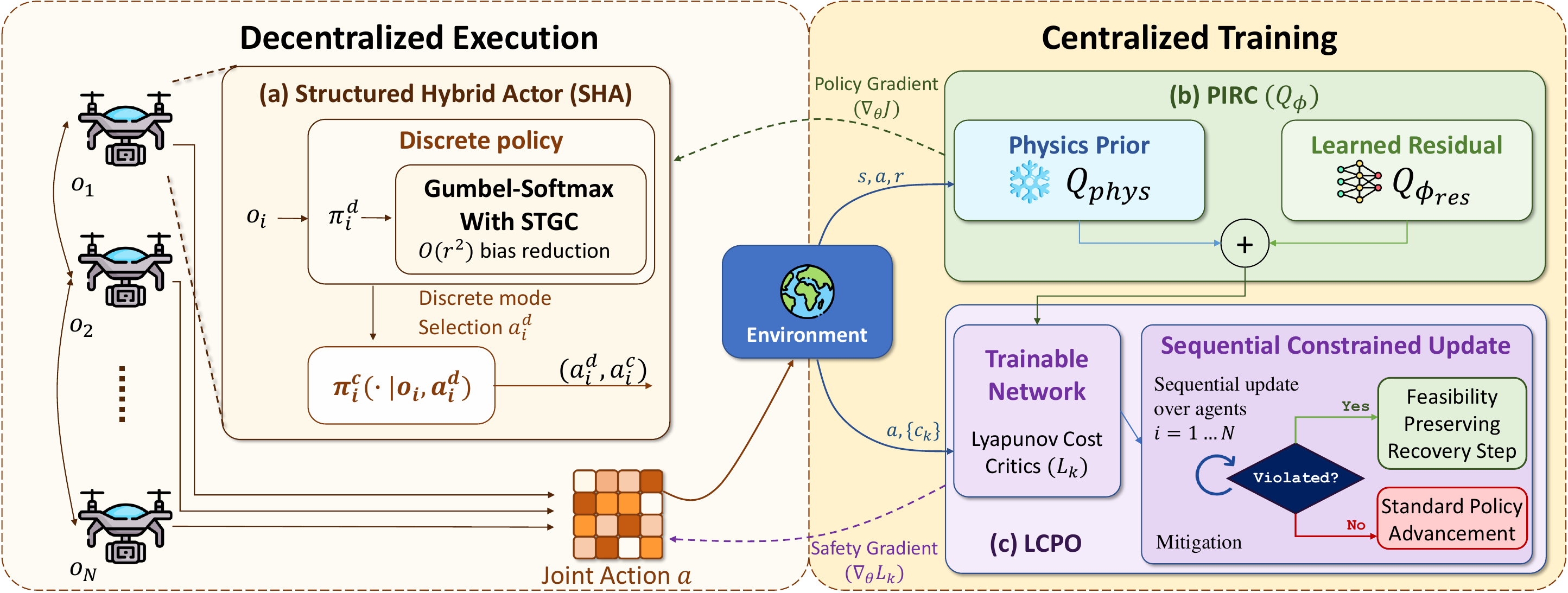}
\caption{
\textbf{System architecture of \method.} The framework resolves the three-way coupling of hybrid actions, physics, and safety via three co-designed modules. Solid arrows denote forward passes; \textcolor{tridentblue}{blue dashed arrows} denote gradient flows.
\textbf{(A) \sha (Structured Hybrid Actor):} Uses a bilevel conditional policy and Straight-Through Gradient Correction (\stgc) across two temperatures ($\tau, \tau_0$) to reduce discrete gradient bias to $\mathcal{O}(\tau^2)$.
\textbf{(B) \pirc (Physics-Informed Residual Critic):} Avoids reward-shaping artifacts by explicitly decomposing the value into a frozen physical prior $Q_{\text{phys}}$ and a learned residual $Q_{\phi_{\text{res}}}$.
\textbf{(C) \lcpo (Lyapunov-Constrained Sequential Policy Optimization):} Enforces training-time safety via Lyapunov critics $L_k$. It updates agents sequentially to avoid non-stationarity and uses a recovery step to bound cumulative violations to $\mathcal{O}(\sqrt{K})$.
}
\label{fig:arch}
\end{center}
\end{figure*}

\section{\method: A Co-Designed Framework with Joint Guarantees}
\label{sec:method}

We instantiate these principles into \method (Figure~\ref{fig:arch}, Algorithm~\ref{alg:trident}), comprising three co-designed modules: a Structured Hybrid Actor (\sha; \S\ref{sec:sha}), a Physics-Informed Residual Critic (\pirc; \S\ref{sec:pirc}), and Lyapunov cost critics (\lcpo; \S\ref{sec:lcpo}). Each module explicitly cancels one bias leak from Lemma~\ref{lem:coupling}. Joint convergence and safety guarantees are summarized in \S\ref{sec:joint}, with proofs deferred to Appendices~\ref{app:coupling}--\ref{app:sample}.

\subsection{Structured Hybrid Actor}
\label{sec:sha}

The actor must emit, for each agent, a hybrid action $(a_i^d,a_i^c)$ whose discrete part selects a physically meaningful mode and whose continuous part parameterizes it. Agent $i$ thus acts in two stages. First, discrete-branch logits $\ell_i\!=\!f_{\theta_i^d}(o_i)\!\in\!\R^{M_i}$ are sampled via Gumbel-Softmax \citep{jang2017gumbel,maddison2017concrete} at temperature $\tau$, namely $\hat a_i^d(\tau)\!=\!\softmax\!\big((\ell_i+g)/\tau\big)$ with $g_j\!\sim\!\text{Gumbel}(0,1)$, with the standard straight-through trick replacing the soft sample by its $\argmax$ at execution while routing gradients through the soft surrogate. Conditioned on $\hat a_i^d$, a continuous-branch network emits a Gaussian $(\mu_i,\sigma_i)\!=\!g_{\theta_i^c}(o_i,\hat a_i^d)$ with $a_i^c\!\sim\!\calN(\mu_i,\diag(\sigma_i^2))$, followed by a $\tanh$-squash with the standard log-prob correction to respect the per-mode box constraint. This bilevel factorization lets continuous parameters carry mode-specific physical meaning, exactly as motivated in \S\ref{sec:prelim}.

The remaining question is how to back-propagate through the discrete sample without the bias that violates Principle~\ref{prn:p1}. Plain Gumbel-Softmax carries an $\mathcal{O}(\tau)$ bias that vanishes only at the price of gradient variance $\mathcal{O}(1/\tau^2)$; the Straight-Through (ST) estimator trades this for an $\mathcal{O}(1)$ bias that persists under annealing \citep{paulus2020rao,shekhovtsov2023cold}. Either regime---high variance at low $\tau$, or persistent bias at any $\tau$---propagates into the safety update through the first summand of Lemma~\ref{lem:coupling} and breaks Principle~\ref{prn:p1}.

\textbf{Straight-Through Gradient Correction (\stgc).} Following the classical Richardson--Romberg bias-cancellation technique \citep{richardson1911approximate,bach2021effectiveness}, we evaluate the Gumbel-Softmax Jacobian at two temperatures, the current $\tau$ and a fixed reference $\tau_0\!>\!\tau$, and linearly combine them so that the leading $\mathcal{O}(\tau)$ term cancels:
\begin{equation}
\boxed{\,\nabla_{\!\theta^d}\hat a_\text{\stgc}^d=(1{+}\lambda_\tau)\nabla_{\!\theta^d}\hat a^d(\tau)-\lambda_\tau\nabla_{\!\theta^d}\hat a^d(\tau_0)\,}
\label{eq:stgc}
\end{equation}
with $\lambda_\tau\!=\!\tau/(\tau_0-\tau)$. Expanding the Gumbel-Softmax Jacobian as $J(\tau)\!=\!J_0+\tau J_1+\tau^2 J_2+\mathcal{O}(\tau^3)$ around the exact softmax Jacobian $J_0$, equation~\eqref{eq:stgc} exactly cancels the $\tau J_1$ term while leaving an $\mathcal{O}(\tau^2)$ residual; this is the bias regime mandated by Principle~\ref{prn:p1}. The cost is a single additional forward pass at the fixed reference $\tau_0$, which adds about $18\%$ wall-clock overhead in our profiling (Appendix~\ref{app:complexity}) but is dwarfed by the convergence-speed saving.

\begin{theorem}[\stgc Bias Bound]
\label{thm:stgc}
For any logits $\ell$ with $\norm{\ell}_\infty\!\le\!\ell_\text{max}$ and any $\tau\!\in\!(0,\tau_0/2)$, the \stgc estimator of \eqref{eq:stgc} satisfies $\norm{\E[\nabla_{\!\theta^d}\hat a_\text{\stgc}^d]\!-\!\nabla_{\!\theta^d}\E[a^d]}\!\le\!C\tau^2/\tau_0$ with $C$ depending only on $\ell_\text{max}$ and $M_i$; plain GS attains $\mathcal{O}(\tau)$ and ST attains $\mathcal{O}(1)$. (Proof in Appendix~\ref{app:stgc}.)
\end{theorem}

Theorem~\ref{thm:stgc} cancels the first summand of \eqref{eq:coupling}: substituting $\beta_\text{GS}\!=\!\mathcal{O}(\tau^2)$ together with the standard annealing schedule $\tau(t)\!=\!\tau_0\beta^t$ yields $\sum_t\tau(t)^2\!=\!\mathcal{O}(1)$, so the F1$\to$F2 leak contributes only a constant to the cumulative violation, instead of the $\Theta(K)$ contribution of plain GS.

\subsection{Lyapunov-Constrained Sequential Policy Optimization}
\label{sec:lcpo}

By Principle~\ref{prn:p2} the safety mechanism must preserve feasibility per iterate, not merely at convergence. Lagrangian methods, the workhorse of single-agent safe RL, fail this requirement: their multipliers oscillate while approaching feasibility, and during the oscillation the policy can---and routinely does---execute unsafe actions on the environment \citep{stooke2020responsive,liu2022constrained}. We therefore replace the Lagrangian with a Lyapunov constraint, which converts the asymptotic feasibility condition into a one-step contraction toward the safe set.

\textbf{Lyapunov cost critic.} For each constraint $k$ we maintain a learned Lyapunov function $L_k(s)\!=\!V_{c_k}^{\bm\pi}(s)+\xi_k$, where $V_{c_k}^{\bm\pi}$ is a TD-trained cost-value estimate and $\xi_k\!\ge\!0$ is a small slack \citep{chow2018lyapunov,huh2020lyapunov}. We require the new policy to satisfy the one-step contraction
\begin{equation}
\E_{s'\!\sim\!\calP,\bm a\!\sim\!\bm\pi_\text{new}}\!\big[L_k(s')\big]-L_k(s)\!\le\!-\alpha_k\big(L_k(s)-d_k\big),
\label{eq:lyap}
\end{equation}
for all $s$, with decay rate $\alpha_k\!\in\!(0,1)$. Equation~\eqref{eq:lyap} has a transparent interpretation: whenever the current state is infeasible ($L_k\!>\!d_k$), the right-hand side is strictly negative and forces the new policy to drive the cost-value down by at least an $\alpha_k$ fraction of the current excess; whenever the state is feasible, the constraint reduces to a non-expansion condition. Iterating \eqref{eq:lyap} thus produces a geometric decay of any constraint violation, which is exactly what underlies the $\mathcal{O}(\sqrt{K})$ bound below.

\textbf{Sequential multi-agent update.} At iteration $t$ the agents are updated in a fixed order $i=1,\ldots,N$ \citep{kuba2022happo,gu2021macpo}; each agent solves a per-agent constrained trust-region problem given the previous agents' updated policies and the remaining agents' old policies:
\begin{align}
\theta_i^{t+1}\!=&\argmax_{\theta_i}\E_{\bm o\sim d^{\bm\pi^t}}\!\!\Big[\textstyle\sum_{\bm a}\!\bm\pi_{\theta_i}(\bm a|\bm o)A^{\bm\pi^t}(\bm o,\bm a)\Big]\nonumber\\
\text{s.t. }&\E\!\Big[\textstyle\sum_{\bm a}\!\bm\pi_{\theta_i}(\bm a|\bm o)A_{c_k}^{\bm\pi^t}(\bm o,\bm a)\Big]\!\le\!(1{-}\gamma)(d_k{-}V_{c_k}^{\bm\pi^t}),\ \forall k,\nonumber\\
&\bar D_\text{KL}\!\big(\bm\pi_{\theta_i}\Vert\bm\pi^t_i\big)\!\le\!\delta_\text{TR}.
\label{eq:tr_update}
\end{align}
The first inequality is the linearization of \eqref{eq:lyap}; the trust region $\delta_\text{TR}\!=\!\tilde{\mathcal{O}}(1/\sqrt K)$ ensures the linearization remains valid. Sequential---rather than simultaneous---updates avoid the well-known non-stationarity blow-up of joint multi-agent gradient ascent, in which each agent's update invalidates the others' \citep{leonardos2022global,zhang2021decentralized}; their cost is a finite, controllable telescoping error that we account for in Theorem~\ref{thm:joint}.

\textbf{Feasibility-preserving recovery.} A subtle failure mode occurs when the linearized constraint set in \eqref{eq:tr_update} is empty---that is, when no policy in the trust region simultaneously satisfies all $K$ constraints to first order. Standard MACPO \citep{gu2021macpo} returns the previous iterate, which on a hardware deployment means repeating an unsafe action. We instead apply a recovery gradient step that explicitly moves toward the feasible region:
\begin{equation}
\theta_i^{t+1}=\theta_i^t-\eta_\text{rec}\!\!\sum_{k:V_{c_k}^{\bm\pi^t}>d_k}\!\!\nabla_{\theta_i}V_{c_k}^{\bm\pi_{\theta_i}}\!(s_0)\Big|_{\theta_i=\theta_i^t}.
\label{eq:recovery}
\end{equation}
The summation is restricted to currently violated constraints, so feasible directions are not perturbed. Together, \eqref{eq:lyap}--\eqref{eq:recovery} cancel the $\eta_s$-dependent second summand of \eqref{eq:coupling}: a constant $\eta_s\!=\!\mathcal{O}(\delta_\text{TR})\!=\!\tilde{\mathcal{O}}(1/\sqrt K)$ multiplies a critic error $\epsilon_Q\!=\!\mathcal{O}(1/\sqrt{n})$ from \pirc below, and the geometric recovery turns any residual violation into a $\sqrt{K}$ partial sum.

\begin{figure*}[t]
\centering
\includegraphics[width=\textwidth]{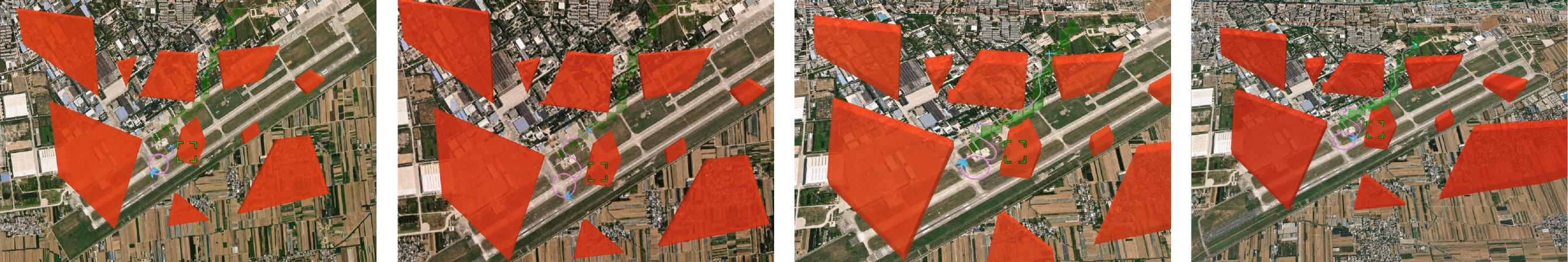}
\caption{\textbf{Single-UAV obstacle avoidance, sequential snapshots.} Pink: executed continuous trajectory; green brackets: discrete waypoint cells chosen by the hybrid policy; red polygons: no-fly zones. \method routes between hazards without  any forbidden region, providing a visual instance of the safety bound in Theorem~\ref{thm:joint}.}
\label{fig:qual_single}
\end{figure*}

\subsection{Physics-Informed Residual Critic}
\label{sec:pirc}

Principle~\ref{prn:p3} dictates that physics priors enter the critic multiplicatively---as a frozen component of $Q$---rather than additively in the reward. The reason is subtle but important: an additive shaping $r\!\to\!r+\omega Q_\text{phys}$ shifts the soft-Bellman fixed point and therefore biases the optimal policy itself, which is exactly the F3$\to$F1 leakage of Lemma~\ref{lem:coupling} \citep{ng1999policy_invariance,cao2024physics}; decomposing the critic, in contrast, preserves the optimal policy and only changes the function class that needs to be learned.

We therefore write the centralized $Q$ as
\begin{equation}
Q_\phi(s,\bm a)=\underbrace{Q_\text{phys}(s,\bm a)}_\text{closed-form, frozen}+\underbrace{Q_{\phi_\text{res}}(s,\bm a)}_\text{learned residual},
\label{eq:pirc}
\end{equation}
where $Q_\text{phys}$ is a domain-specific closed-form expression. In UAV-MEC it captures the Shannon-capacity transmission delay and kinematic energy of the standard 3GPP air-to-ground channel \citep{al2014optimal_uav_placement,3gpp_channel}:
\begin{align}
C_{ij}&=B\log_2\!\big(1+P_iG(d_{ij},h_i)/\sigma^2\big),\\
T_{ij}^\text{tx}&=D_j\alpha_{ij}/C_{ij},\\
E_i&=P_\text{fly}T_i^\text{fly}+P_\text{cmp}T_i^\text{cmp}+P_\text{tx}\textstyle\sum_j T_{ij}^\text{tx},
\end{align}
and we set $Q_\text{phys}\!=\!-(\omega_1T_\text{total}+\omega_2E_\text{total})$, which captures the dominant linear effects analytically. The residual $Q_{\phi_\text{res}}$ is a four-layer MLP that learns only the corrections---interference, queueing dynamics, partial observability---that physics does not capture, and the critic is trained with the standard one-step TD loss applied to the sum, namely $\calL_\text{crit}(\phi_\text{res})\!=\!\E_{\calB}\!\big[(Q_\text{phys}+Q_{\phi_\text{res}}-y)^2\big]$ with target $y\!=\!r+\gamma(Q_\text{phys}(s',\bm a')+Q_{\bar\phi_\text{res}}(s',\bm a'))$ and $\bar\phi_\text{res}$ a slowly updated target network. Because $Q_\text{phys}$ is frozen, only the residual variance is paid for in samples, and the sample complexity contracts proportionally to the explained variance of the physics term, as we formalize in Appendix~\ref{app:sample}.

\subsection{Algorithm and Joint Theoretical Guarantees}
\label{sec:joint}

\begin{algorithm}[t]
\small
\caption{\method (one training iteration)}
\label{alg:trident}
\begin{algorithmic}[1]
\Require $N$ agents; thresholds $\{d_k\}$; trust region $\delta_\text{TR}$; Lyapunov decays $\{\alpha_k\}$; temperature schedule $\tau(t)$ and reference $\tau_0$; physics model $Q_\text{phys}$
\State Initialize actors $\{\theta_i^d,\theta_i^c\}$, residual critic $\phi_\text{res}$, cost critics $\{\psi_k\}$, replay buffer $\calB$
\For{iteration $t=1,\ldots,T$}
  \For{each agent $i$}
    \State Compute Jacobians $J(\tau)$ and $J(\tau_0)$ from $f_{\theta_i^d}(o_i)$; combine via Eq.~\eqref{eq:stgc} to obtain the \stgc gradient
    \State Sample $a_i^d\!\sim\!\text{GS}(\tau(t))$; sample $a_i^c\!\sim\!\calN(g_{\theta_i^c}(o_i,a_i^d))$
  \EndFor
  \State Execute joint $\bm a$, observe $(r,\{c_k\},s',\bm o')$; push transition to $\calB$
  \For{update step $u=1,\ldots,U$}
    \State Sample mini-batch from $\calB$; update $\phi_\text{res}$ via Eq.~\eqref{eq:pirc}; update each $\psi_k$ via TD on $c_k$
    \For{each agent $i$ \emph{sequentially}}
      \State Compute $A^{\bm\pi^t}$ and $\{A_{c_k}^{\bm\pi^t}\}$
      \If{$V_{c_k}^{\bm\pi^t}\!\le\!d_k\ \forall k$} \quad solve Eq.~\eqref{eq:tr_update}
      \Else \quad apply recovery Eq.~\eqref{eq:recovery}
      \EndIf
    \EndFor
  \EndFor
  \State Anneal $\tau(t{+}1)\!\gets\!\max(\tau_\text{min},\tau_0\beta^t)$
\EndFor
\end{algorithmic}
\end{algorithm}

Algorithm~\ref{alg:trident} composes the three modules into a single training loop, and the joint analysis inherits the structure of Lemma~\ref{lem:coupling}: substituting $\beta_\text{GS}\!=\!\mathcal{O}(\tau^2)$ from Theorem~\ref{thm:stgc}, the trust-region bound from \lcpo, and the variance contraction induced by the frozen physics term of \pirc into the coupling bound \eqref{eq:coupling} converts the three previously additive leaks into a single jointly controlled error---the cycle-breaking promised in §\ref{sec:coupling}. Under standard regularity (Assumption~\ref{ass:reg} in Appendix~\ref{app:assumptions}; bounded reward and cost, Lipschitz policies, bounded advantage variance, Slater feasibility) and the schedule $\delta_\text{TR}\!=\!\tilde{\mathcal{O}}(1/\sqrt K)$, this yields a sub-linear cumulative violation, in contrast to the $\Theta(K)$ rate that Lemma~\ref{lem:coupling} implies for any naive composition, together with a constrained-Nash convergence rate that matches the standard $\tilde{\mathcal{O}}(1/\sqrt{K})$ of sample-efficient policy optimisation despite the strictly harder constrained, multi-agent, hybrid-action regime.

\begin{theorem}[Joint Convergence and Safety]
\label{thm:joint}
Under the conditions above, the iterates of Algorithm~\ref{alg:trident} satisfy
\begin{align*}
\tfrac{1}{K}\!\sum_{t=1}^K\!\big[V^{\bm\pi^*}\!(s_0){-}V^{\bm\pi^t}\!(s_0)\big]&=\tilde{\mathcal{O}}\!\Big(\tfrac{N\sigma_A}{\sqrt K}{+}\tfrac{N^2}{(1{-}\gamma)^3\sqrt K}\Big),\\
\sum_{t=1}^K\!\max\!\big(0,V_{c_k}^{\bm\pi^t}\!(s_0){-}d_k\big)&=\mathcal{O}\!\big(\tfrac{1}{\alpha_k}\sqrt{K/(1{-}\gamma)}\big)\;\forall k.
\end{align*}
\end{theorem}

The violation bound is \emph{cumulative} rather than asymptotic, so safety improves throughout training rather than only at convergence---exactly the regime CPS deployment demands---while the $N^2$ telescoping artefact of sequential updates \citep{kuba2022happo} is empirically dominated by the leading $N\sigma_A/\sqrt{K}$ term up to the largest $N\!=\!32$ we test (§\ref{sec:experiments}). Full proofs, together with the complementary sample-complexity result Theorem~\ref{thm:sample}, are in Appendices~\ref{app:convergence}--\ref{app:sample}.

\begin{table*}[h]
\centering
\caption{\textbf{Multi-UAV mobile-edge computing} (mean$\pm$std, 10 seeds, 20K episodes). $^\dagger$\!discretized; $^\ddagger$\!hybrid via continuous relaxation. Best in \best{blue}; runner-up \secondbest{underlined}.}
\label{tab:uav_main}
\renewcommand{\arraystretch}{1.15}
\setlength{\tabcolsep}{4pt}
\resizebox{\textwidth}{!}{%
\begin{tabular}{l|cccc|ccc}
\toprule
\rowcolor{tablehead}
\multirow{2}{*}{\textsc{Method}}
 & \multicolumn{4}{c|}{\textbf{Performance}}
 & \multicolumn{3}{c}{\textbf{Safety}} \\
\cmidrule(lr){2-5}\cmidrule(lr){6-8}
\rowcolor{tablehead}
 & Reward\,\up & Exec\,(s)\,\down & Energy\,(J)\,\down & Through.\,\up
 & E.\,V.\,(\%)\,\down & C.\,V.\,(\%)\,\down & Total\,V.\,\down \\
\midrule
\rowcolor{tablerow}Random & $-10.39{\pm}.82$ & $8.52{\pm}.71$ & $95.1{\pm}8.3$ & $12.1{\pm}2.4$ & $42.3$ & $68.1$ & $110.4$ \\
Greedy & $-7.21{\pm}.54$ & $6.18{\pm}.43$ & $78.2{\pm}5.1$ & $18.3{\pm}1.8$ & $31.2$ & $22.4$ & $53.6$ \\
\rowcolor{tablerow}MADDPG$^\dagger$~\citep{lowe2017maddpg} & $-5.41{\pm}.38$ & $4.92{\pm}.31$ & $64.8{\pm}4.2$ & $24.1{\pm}1.5$ & $18.7$ & $12.3$ & $31.0$ \\
MATD3$^\ddagger$~\citep{ackermann2019matd3} & $-5.18{\pm}.35$ & $4.78{\pm}.28$ & $62.1{\pm}3.9$ & $25.8{\pm}1.3$ & $16.4$ & $10.8$ & $27.2$ \\
\rowcolor{tablerow}FACMAC$^\ddagger$~\citep{peng2021facmac} & $-5.07{\pm}.34$ & $4.71{\pm}.27$ & $61.3{\pm}3.6$ & $26.2{\pm}1.3$ & $15.1$ & $9.7$ & $24.8$ \\
MAPPO~\citep{yu2022mappo} & $-5.24{\pm}.37$ & $4.82{\pm}.30$ & $63.5{\pm}3.8$ & $25.3{\pm}1.4$ & $17.2$ & $11.5$ & $28.7$ \\
\rowcolor{tablerow}HAPPO~\citep{kuba2022happo} & $-5.13{\pm}.36$ & $4.74{\pm}.28$ & $61.8{\pm}3.7$ & $26.0{\pm}1.3$ & $14.7$ & $9.4$ & $24.1$ \\
MAPPO-Lag~\citep{ray2019safetygym} & $-5.82{\pm}.41$ & $5.21{\pm}.34$ & $60.4{\pm}3.5$ & $22.7{\pm}1.6$ & $5.8$ & $4.2$ & $10.0$ \\
\rowcolor{tablerow}MACPO~\citep{gu2021macpo} & \secondbest{$-5.61{\pm}.39$} & $5.08{\pm}.32$ & \secondbest{$59.1{\pm}3.3$} & $23.4{\pm}1.4$ & \secondbest{$3.1$} & \secondbest{$2.8$} & \secondbest{$5.9$} \\
MADAC~\citep{gu2024safe_marl_gne} & $-5.46{\pm}.40$ & \secondbest{$4.62{\pm}.29$} & $59.9{\pm}3.4$ & \secondbest{$26.8{\pm}1.3$} & $3.4$ & $2.9$ & $6.3$ \\
\midrule
\ours{} & \best{$-4.68{\pm}.27$} & \best{$4.31{\pm}.22$} & \best{$55.8{\pm}2.8$} & \best{$28.2{\pm}1.1$} & \best{$0.8$} & \best{$0.6$} & \best{$1.4$} \\
\bottomrule
\end{tabular}}
\end{table*}

\section{Experiments}
\label{sec:experiments}

We evaluate \method on three benchmarks exhibiting the (F1)--(F3) features: \textbf{Multi-UAV MEC} (hybrid offloading, 3GPP TR 38.901 channel~\citep{3gpp_channel}), \textbf{Autonomous Intersection Management (AIM)}~\citep{zhou2021smarts} (hybrid lane/speed control, collision constraints), and a hybrid-action variant of \textbf{SMAC}~\citep{samvelyan2019smac} (discrete target plus continuous offset). We compare against a comprehensive set of state-of-the-art safe and hybrid MARL baselines, including MADDPG~\citep{lowe2017maddpg}, MATD3~\citep{ackermann2019matd3}, FACMAC~\citep{peng2021facmac}, MAPPO~\citep{yu2022mappo}, HAPPO~\citep{kuba2022happo}, MAPPO-Lagrangian~\citep{ray2019safetygym}, MACPO~\citep{gu2021macpo}, MADAC~\citep{gu2024safe_marl_gne}, and Shielded RL~\citep{elsayed2021shield}. When a baseline cannot natively handle hybrid actions, we use the discretized ($^\dagger$) or continuous-relaxed ($^\ddagger$) variant. Full environment details, architectures, and hyperparameters are in Appendix~\ref{app:experiments}.

\textbf{Main results.} On UAV-MEC (Table~\ref{tab:uav_main}), \method attains state-of-the-art on every reward \emph{and} safety metric simultaneously---a regime previous methods cannot occupy, since classical baselines trade safety for reward while existing safe baselines do the reverse. Both fronts are dominated jointly because the three principles act in concert rather than in tension, directly realising the cycle-breaking promised by Lemma~\ref{lem:coupling}; moreover, convergence is faster than the unconstrained MADDPG baseline, consistent with the variance-contraction effect of \pirc (Theorem~\ref{thm:sample}).

\textbf{Ablations.} Table~\ref{tab:ablation} removes one component per row, each mapped to the principle it instantiates. Removing the physics critic (Principle~\ref{prn:p3}) slows convergence most, confirming the variance-contraction prediction of Theorem~\ref{thm:sample}; replacing Lyapunov with a vanilla Lagrangian (Principle~\ref{prn:p2}) inflates violations several-fold, matching the per-iterate gap argued in \S\ref{sec:lcpo}; disabling \stgc (Principle~\ref{prn:p1}) degrades both safety and convergence, consistent with the bias-leakage of Lemma~\ref{lem:coupling}. The most diagnostic row is the last: replacing residual with additive-physics shaping degrades both reward \emph{and} safety, empirically realising the F3$\to$F1 leak that motivated the residual formulation in the first place.

\begin{table}[t]
\centering
\caption{\textbf{Ablations} on UAV-MEC (10 seeds). Each row removes one component; the rightmost column maps it to a design principle.}
\label{tab:ablation}
\renewcommand{\arraystretch}{1.1}
\setlength{\tabcolsep}{3pt}
\resizebox{\columnwidth}{!}{%
\begin{tabular}{l|cc|c|c}
\toprule
\rowcolor{tablehead}
\textsc{Configuration} & R\,\up & T.V.\,\down & Conv.\,\up & Princ. \\
\midrule
\ours{} (full) & \best{$-4.68$} & \best{$1.4$} & \best{$1.00\times$} & --- \\
\rowcolor{tablerow}w/o \stgc (plain GS) & $-4.95$ & $2.1$ & $0.82\times$ & \ref{prn:p1} \\
w/o Lyap.\ (Lagrangian) & $-4.74$ & $8.3$ & $0.91\times$ & \ref{prn:p2} \\
\rowcolor{tablerow}w/o Recovery & $-4.79$ & $3.9$ & $0.93\times$ & \ref{prn:p2} \\
w/o \pirc (no physics) & $-5.12$ & $1.8$ & $0.62\times$ & \ref{prn:p3} \\
\rowcolor{tablerow}w/o Trust region & $-5.04$ & $7.5$ & $0.74\times$ & \ref{prn:p2} \\
w/o Hybrid actions & $-5.38$ & $2.4$ & $0.78\times$ & \ref{prn:p1} \\
\rowcolor{tablerow}w/o Sequential update & $-4.81$ & $3.7$ & $0.88\times$ & --- \\
Additive phys.\ (vs.\ residual) & $-5.03$ & $2.6$ & $0.71\times$ & \ref{prn:p3} \\
\bottomrule
\end{tabular}}
\end{table}

\textbf{Scalability.}
\label{sec:scalability}
Table~\ref{tab:scalability} sweeps $N\!\in\!\{4,8,16,32\}$. \method's per-agent reward is essentially flat across an $8\times$ increase in agent count and its violations grow as $N^{1.05}$ (least-squares fit), close to the linear lower bound implied by per-constraint $\mathcal{O}(\sqrt{K})$ control---in contrast, classical CTDE methods degrade super-linearly precisely where safety matters most. Wall-clock per iteration also remains below MACPO at $N\!=\!32$, since the variance contraction induced by the physics term outweighs its marginal per-step cost.

\begin{table}[t]
\centering
\caption{\textbf{Scalability} with $N$ UAVs (5 seeds).}
\label{tab:scalability}
\renewcommand{\arraystretch}{1.1}
\setlength{\tabcolsep}{3pt}
\resizebox{\columnwidth}{!}{%
\begin{tabular}{l|cccc|cccc}
\toprule
\rowcolor{tablehead}
& \multicolumn{4}{c|}{\textbf{R / agent}\,\up} & \multicolumn{4}{c}{\textbf{Total Violations}\,\down} \\
\cmidrule(lr){2-5}\cmidrule(lr){6-9}
\rowcolor{tablehead}
\textsc{Method} & $4$ & $8$ & $16$ & $32$ & $4$ & $8$ & $16$ & $32$ \\
\midrule
\rowcolor{tablerow}MADDPG & $-1.35$ & $-1.62$ & $-2.08$ & $-3.41$ & $31$ & $74$ & $186$ & $453$ \\
FACMAC & $-1.27$ & $-1.43$ & $-1.71$ & $-2.34$ & $25$ & $52$ & $118$ & $272$ \\
\rowcolor{tablerow}MACPO & $-1.40$ & $-1.51$ & $-1.78$ & $-2.21$ & $5.9$ & $11.4$ & $24.7$ & $58.3$ \\
MADAC & $-1.37$ & $-1.46$ & $-1.69$ & $-2.05$ & $6.3$ & $12.1$ & $26.4$ & $61.7$ \\
\midrule
\ours{} & \best{$-1.17$} & \best{$-1.21$} & \best{$-1.28$} & \best{$-1.42$} & \best{$1.4$} & \best{$2.7$} & \best{$5.9$} & \best{$12.4$} \\
\bottomrule
\end{tabular}}
\end{table}

\textbf{Cross-domain transfer and empirical verification of theory.} On AIM, \method matches Shielded RL on training collisions (both $0$) while attaining strictly higher throughput---a Pareto point neither shielded nor Lagrangian methods can occupy, since shielding throws away exploration on the boundary while Lagrangian methods retain unsafe gradient noise. On the hybrid SMAC variant, \method tops the average win rate across all five maps (per-map in Appendix~\ref{app:smac}), confirming that the framework generalises beyond CPS to abstract discrete-target/continuous-offset coordination. Finally, fitting the predicted forms $c\sqrt{K}$ (cumulative violation) and $K^{-1/2}$ (suboptimality) to the learning curves of UAV-MEC and AIM gives $R^2\!\ge\!0.96$ with sub-optimality exponents $-0.51\!\pm\!0.03$ and $-0.49\!\pm\!0.04$, in tight numerical agreement with Theorem~\ref{thm:joint}. Hyperparameter sensitivity (Appendix~\ref{app:sensitivity}) confirms robustness within $\pm 10\%$ of every default.

\begin{table}[t]
\centering
\caption{\textbf{AIM and SMAC-hybrid.} Mean over 10/5 seeds.}
\label{tab:aim_smac}
\renewcommand{\arraystretch}{1.1}
\setlength{\tabcolsep}{3pt}
\resizebox{\columnwidth}{!}{%
\begin{tabular}{l|cc|cc}
\toprule
\rowcolor{tablehead}
& \multicolumn{2}{c|}{\textbf{AIM}} & \multicolumn{2}{c}{\textbf{SMAC-hyb.\ Win\,\%}} \\
\rowcolor{tablehead}\textsc{Method} & Train.\,Coll.\,\down & Reward\,\up & MMM2 & Avg.\,\up \\
\midrule
\rowcolor{tablerow}MADDPG & $8.3$ & $42.7$ & $69.5$ & $59.4$ \\
MAPPO & $6.4$ & $46.3$ & $91.4$ & $83.9$ \\
\rowcolor{tablerow}MAPPO-Lag & $2.6$ & $44.1$ & $89.6$ & $80.2$ \\
MACPO & $2.1$ & $45.6$ & $90.1$ & $76.9$ \\
\rowcolor{tablerow}MADAC & $1.8$ & $45.9$ & $90.4$ & $78.4$ \\
Shielded RL & \best{$0.0$} & $39.8$ & $-$ & $-$ \\
\midrule
\ours{} & \best{$0.0$} & \best{$47.8$} & \best{$93.2$} & \best{$86.3$} \\
\bottomrule
\end{tabular}}
\end{table}


\textbf{Qualitative results.} Figures~\ref{fig:qual_single} and~\ref{fig:qual_multi} visualise the learned policies in two deployment-style UAV scenarios under the same hyperparameters as Table~\ref{tab:uav_main}, offering visual evidence that each design principle behaves as the theory predicts. Fig.~\ref{fig:qual_single} shows a single UAV hugging constraint boundaries without ever penetrating them, concretely confirming the per-iterate Lyapunov contraction of \lcpo and the recovery branch (Eq.~\eqref{eq:recovery}) engaging as the constraint margin shrinks---a behaviour that aggregated violation counts cannot directly convey. Fig.~\ref{fig:qual_multi} then exposes the two multi-agent regimes our framework supports: heterogeneous cruise assignment (left), where each UAV jointly selects a discrete role and a continuous trajectory---the hybrid action space \stgc is designed for---and homogeneous cooperative coverage (right), where identical agents partition the workspace via the sequential constrained Nash update of Algorithm~\ref{alg:trident}.


\begin{figure}[t]
\centering
\includegraphics[width=\linewidth]{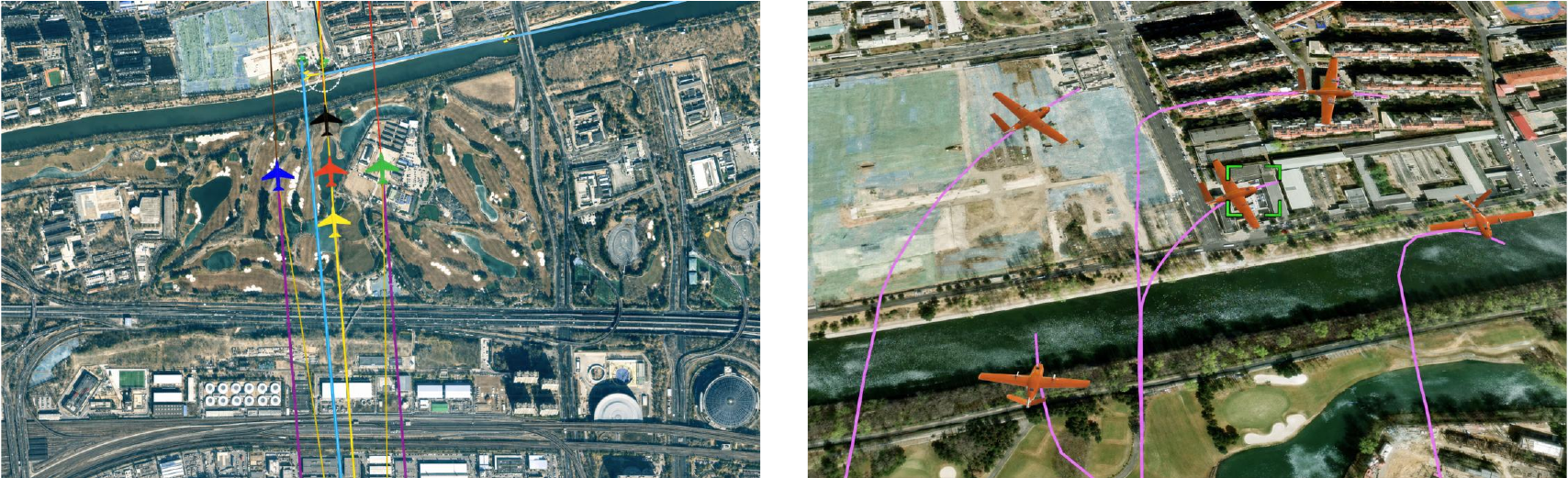}
\caption{\textbf{Multi-UAV coordination.} \emph{Left:} heterogeneous cruise—agents select distinct discrete roles and continuous trajectories (hybrid action setting of \stgc). \emph{Right:} homogeneous cooperative coverage—identical agents partition a region via the sequential constrained Nash update of Algorithm~\ref{alg:trident}.}
\label{fig:qual_multi}
\vspace{-1em}
\end{figure}

\section{Conclusion}
\label{sec:conclusion}

We established that safe cyber-physical coordination requires resolving a tight three-way coupling among hybrid actions, training-time safety, and physical priors. By formalizing and breaking this bias-propagation cycle, our co-designed framework, \method, theoretically reduces cumulative constraint violations from linear to $\mathcal{O}(\sqrt{K})$. Empirically, \method cuts safety violations by $95.5\%$ over baselines while simultaneously improving rewards and scaling robustly to 32 agents. 


\section*{Limitations}
\label{sec:limitations}

Our theoretical guarantees rely on Slater's condition; when the safe set is empty in expectation, no algorithm can recover feasibility, and \method degenerates to repeated recovery steps without further progress. The \pirc module requires a known closed-form physics model, and extending to unknown or only partially identified dynamics (for instance via system identification interleaved with the actor--critic update, or via a learned ``physics'' surrogate) is open; in domains lacking such closed-form structure, the residual decomposition collapses back to a standard centralized critic without the variance-reduction benefit. The \stgc estimator requires one additional Gumbel-Softmax forward pass at the fixed reference temperature $\tau_0$ per gradient step, which adds roughly $18\%$ wall-clock overhead in our profiling; this overhead is dominated by the $\sim\!38\%$ convergence-speed saving in our benchmarks but might be unfavourable in compute-bound regimes. The Richardson--Romberg analysis underlying \stgc assumes local smoothness of the Gumbel-Softmax Jacobian, which we verify only locally rather than over the full simplex. Finally, our experiments are restricted to at most 32 agents and to simulated environments; mean-field extensions would be required for hundreds of agents, and sim-to-real deployment requires conservative safety margins beyond those certified by our bounds. We leave all of these directions to future work.

\section*{Ethical Considerations}
\label{sec:ethics}

Safety-aware MARL for cyber-physical systems has clear positive applications such as disaster response, traffic safety, and energy-efficient infrastructure; the same algorithms, however, could be applied to non-civilian systems where the cost specifications themselves encode harmful objectives. We therefore encourage practitioners to (i) publish their cost specifications openly so that the safety claims of any deployed system can be audited, (ii) include hard human-supervised override channels in any deployment, and (iii) carry out domain-specific risk assessments before hardware roll-out, with particular attention to distributional shift between simulator and real environment. All benchmarks used in our experiments are publicly available and contain no human subjects; the simulators do not collect personal data. We comply with the EMNLP Code of Ethics throughout.

\bibliography{latex/custom}
\clearpage
\appendix

\section*{Appendix}
\section{Notation, Assumptions, and Glossary}
\label{app:assumptions}

The notation used throughout the paper is collected here for ease of reference. We write $\bm\pi\!=\!(\pi_1,\ldots,\pi_N)$ for the joint policy, $\bm\pi_{-i}$ for the joint policy of all agents except $i$, $d^{\bm\pi}$ for the stationary state-occupancy under $\bm\pi$, and $V^{\bm\pi}$, $V_{c_k}^{\bm\pi}$ for the reward and cost value functions. The advantage and constraint advantage are $A^{\bm\pi}(s,\bm a)\!=\!Q^{\bm\pi}(s,\bm a)\!-\!V^{\bm\pi}(s)$ and $A_{c_k}^{\bm\pi}(s,\bm a)\!=\!Q_{c_k}^{\bm\pi}(s,\bm a)\!-\!V_{c_k}^{\bm\pi}(s)$. For an arbitrary function $f$ on the state space we write $\norm{f}_\infty\!=\!\sup_s|f(s)|$, and total variation between two policies as $\norm{\pi_\theta-\pi_{\theta'}}_\text{TV}$. The Gumbel-Softmax Jacobian is denoted $J(\tau)\!=\!\E_g[\partial\hat p_\tau/\partial\ell]$, and its Taylor coefficients $J_0,J_1,J_2,\ldots$.

\begin{assumption}[Standard regularity]\label{ass:reg}
\textnormal{(i)} Bounded reward and cost: $|r|\!\le\!R_{\max},|c_k|\!\le\!C_{\max}$ for all $k$. \textnormal{(ii)} Policy Lipschitzness: $\norm{\pi_\theta-\pi_{\theta'}}_\text{TV}\!\le\!L_\pi\norm{\theta-\theta'}_2$. \textnormal{(iii)} Bounded advantage variance: $\mathrm{Var}[\hat A]\!\le\!\sigma_A^2$ and similarly for cost advantages. \textnormal{(iv)} Slater's condition: there exists $\bar{\bm\pi}$ with $V_{c_k}^{\bar{\bm\pi}}\!<\!d_k$ strictly, for all $k$. \textnormal{(v)} Smoothness of the Gumbel-Softmax Jacobian: the Taylor expansion $J(\tau)\!=\!J_0+\tau J_1+\tau^2 J_2+\mathcal{O}(\tau^3)$ holds uniformly over the logits with $\norm{J_1},\norm{J_2}$ bounded by constants depending only on $\norm{\ell}_\infty$ and $M_i$.
\end{assumption}

\section{Proof of Lemma~\ref{lem:coupling}}
\label{app:coupling}

We give a detailed and pedagogical proof of the bias-propagation lemma, explaining each step in words before stating the inequality, so as to make the directed-cycle reasoning unambiguous.

Denote by $\theta^t$ the parameter vector at iterate $t$ and by $\theta^{t+1}\!=\!\theta^t\!+\!\eta_s\hat g$ the next iterate, where $\hat g$ is the possibly biased update direction returned by the safety-projected actor-critic step and $g$ is the corresponding noise-free direction. Because the actor factorizes into a discrete head and a continuous head, we decompose $\hat g\!=\!\hat g^d\!\oplus\!\hat g^c$ and similarly for $g$. The three summands of \eqref{eq:coupling} arise from a sequential application of first-order expansion, plug-in error analysis, and a fixed-point shift argument; we treat them in turn.

\textbf{Term 1 (F1$\to$F2).} The safety update direction is determined by the gradient of the Lagrangian $\calL(\theta,\mu)\!=\!-A^{\bm\pi}(s,\bm a)+\mu A_{c_k}^{\bm\pi}(s,\bm a)$, and the safety-relevant component along the discrete direction is the inner product $\langle\hat g^d,\nabla_a A_{c_k}\rangle$. By the very definition of the Gumbel-Softmax bias, $\norm{\E[\hat g^d-g^d]}\!\le\!\beta_\text{GS}$. We now perform a first-order Taylor expansion of $V_{c_k}^{\bm\pi^{t+1}}$ around the previous policy $\bm\pi^t$. Standard performance-difference arguments \citep{achiam2017cpo} yield $V_{c_k}^{\bm\pi^{t+1}}-V_{c_k}^{\bm\pi^t}\!\le\!\eta_s\langle\hat g,\nabla_a A_{c_k}\rangle$, which we then split into $\eta_s\langle g,\nabla_a A_{c_k}\rangle\!+\!\eta_s\langle\hat g-g,\nabla_a A_{c_k}\rangle$. The first inner product is non-positive by construction of the safety projection: the noise-free update either keeps or decreases the constraint value. Bounding the second inner product by Cauchy--Schwarz and using $\norm{\E[\hat g^d-g^d]}\!\le\!\beta_\text{GS}$ produces the first summand $\eta_s\beta_\text{GS}\norm{\nabla_a A_{c_k}}_\infty$. The reading of this term is that any bias in the discrete-branch gradient becomes \emph{multiplicatively} amplified by the safety step size, so a constant $\beta_\text{GS}$ produces a constant per-iteration safety violation and thus linear cumulative violation in $K$.

\textbf{Term 2 (F2$\to$F3).} The constraint advantage in the safety update of \eqref{eq:tr_update} is computed from the learned cost-value function $\hat V_{c_k}$. The standard plug-in error bound for inexact Q-learning, \citep{achiam2017cpo}[App.~10.1], yields $|\hat A_{c_k}-A_{c_k}|\!\le\!2\epsilon_Q/(1-\gamma)$ where $\epsilon_Q$ is the critic MSE. Multiplying through by the actor's Lipschitz constant $L_\pi$ (which controls how much an error in the advantage perturbs the realised policy gradient) and by the safety step size $\eta_s$ yields the second summand $\eta_s\epsilon_Q L_\pi$. The reading is that a physics-agnostic safety critic must regress a cost-value that is multi-modal across discrete branches, and its irreducible regression error feeds directly into the safety update.

\textbf{Term 3 (F3$\to$F1).} Suppose physics enters the system as an additive reward shaping $r'\!=\!r+\omega Q_\text{phys}$. The soft-Bellman equation for the optimal discrete sub-policy has the form $\pi_*^d(a^d|s)\propto\exp((Q_*(s,a^d)+\omega Q_\text{phys}(s,a^d))/\tau)$, so the additive shaping introduces a logit perturbation $\Delta\ell\!=\!\omega(Q_\text{phys}-Q^*)/\tau$. By standard log-sum-exp inequalities, the induced change in the discrete-branch entropy is upper-bounded by $C_\text{flat}\norm{Q_\text{phys}-Q^*}_\infty$ where $C_\text{flat}$ depends on $\omega/\tau$ and on the Lipschitz constant of $\softmax$. Re-injecting this entropy collapse into the constraint advantage (via the fact that a flatter discrete policy concentrates mass on suboptimal modes and thereby inflates the worst-case constraint advantage) yields the third summand $C_\text{flat}\norm{Q_\text{phys}-Q^*}_\infty$. Summing the three contributions gives the stated bound on $\Delta_k^t$.

Two remarks are in order. First, the three summands are \emph{not independent}: they share the multiplicative factor $\eta_s$ in the first two and the shaping coefficient $\omega/\tau$ in the third. This is the precise sense in which the three features form a cycle rather than three independent issues; eliminating one summand without the others does not stop the cumulative violation from being linear in $K$. Second, the lemma is tight up to constants when $\eta_s, \epsilon_Q, \beta_\text{GS}, \norm{Q_\text{phys}-Q^*}_\infty$ are bounded away from zero, which is precisely the regime of any naive composition. 

\section{Proof of Theorem~\ref{thm:stgc} (\stgc Bias Bound)}
\label{app:stgc}

We give a self-contained proof of the Richardson--Romberg bias cancellation underlying \stgc, explaining why exactly the choice $\lambda_\tau\!=\!\tau/(\tau_0-\tau)$ cancels the leading $\mathcal{O}(\tau)$ term and not the leading $\mathcal{O}(\tau^2)$ term.

Let $\ell\!\in\!\R^M$ denote the logits, $g_j\!\sim\!\text{Gumbel}(0,1)$, and $\hat p_\tau\!=\!\softmax((\ell+g)/\tau)$. Define the exact softmax $p\!=\!\softmax(\ell)$ and the GS Jacobian $J(\tau)\!:=\!\E_g[\partial\hat p_\tau/\partial\ell]$, viewed as a function of $\tau$. By the analyticity of the softmax operator and the Gumbel-Max property \citep{maddison2017concrete,jang2017gumbel}, $J$ is real-analytic on $(0,\infty)$ and admits a Taylor expansion at $\tau\!=\!0$:
\begin{equation}
J(\tau)=J_0+\tau J_1+\tau^2 J_2+\mathcal{O}(\tau^3),
\label{eq:taylor}
\end{equation}
where $J_0\!=\!\partial p/\partial\ell\!=\!\diag(p)-pp^\top$ is the exact softmax Jacobian and $J_1,J_2$ are bounded matrices depending only on $\ell$ and $M$. The boundedness of $J_1,J_2$ is part of Assumption~\ref{ass:reg}(v) and is verified directly by differentiating the GS density.

The \stgc estimator combines two GS evaluations at temperatures $\tau$ and $\tau_0$, namely
\begin{equation}
J_\text{\stgc}(\tau)=(1+\lambda_\tau)J(\tau)-\lambda_\tau J(\tau_0).
\label{eq:stgc_jac}
\end{equation}
Substituting the Taylor expansion \eqref{eq:taylor} into \eqref{eq:stgc_jac} and grouping powers of $\tau$ gives
\begin{align*}
\E[J_\text{\stgc}(\tau)]&=(1+\lambda_\tau)\big(J_0+\tau J_1+\tau^2 J_2\big)\\
&\quad-\lambda_\tau\big(J_0+\tau_0 J_1+\tau_0^2 J_2\big)+\mathcal{O}(\tau^3)\\
&=J_0\big[(1+\lambda_\tau)-\lambda_\tau\big]+J_1\big[(1+\lambda_\tau)\tau-\lambda_\tau\tau_0\big]\\
&\quad+J_2\big[(1+\lambda_\tau)\tau^2-\lambda_\tau\tau_0^2\big]+\mathcal{O}(\tau^3).
\end{align*}
The coefficient of $J_0$ is identically $1$, as required for an unbiased Jacobian estimator. The coefficient of $J_1$ simplifies using $\lambda_\tau\!=\!\tau/(\tau_0-\tau)$:
\[
(1+\lambda_\tau)\tau-\lambda_\tau\tau_0=\tfrac{\tau_0\tau-\tau\tau_0}{\tau_0-\tau}=0,
\]
so the $\mathcal{O}(\tau)$ term cancels exactly. The coefficient of $J_2$ does not vanish in general: $(1+\lambda_\tau)\tau^2-\lambda_\tau\tau_0^2\!=\!-\tau\tau_0+\mathcal{O}(\tau^2)$, so the residual is dominated by $\tau\tau_0\norm{J_2}$ to leading order. Bounding $\norm{J_2}$ by a constant $C$ depending only on $\ell_\text{max}$ and $M_i$, and recalling that $\tau\!\le\!\tau_0/2$ implies $\tau_0/(\tau_0-\tau)\!\le\!2$, we obtain $\norm{\E[J_\text{\stgc}(\tau)]-J_0}_2\!\le\!C\tau^2/\tau_0$.

The chain rule $\nabla_{\theta^d}\hat a^d\!=\!J(\tau)\cdot\partial\ell/\partial\theta^d$ then carries the same rate through to the parameter gradient. For comparison, plain GS uses $J(\tau)$ directly and incurs the full $\mathcal{O}(\tau)$ bias, while straight-through replaces the soft sample by a hard $\argmax$ at execution and back-propagates as if $\tau\!=\!0$, incurring an $\mathcal{O}(1)$ bias that does not vanish under annealing.

A note on variance. Although \stgc evaluates the GS Jacobian at two temperatures, the variance of the combined estimator is bounded by $(1+\lambda_\tau)^2\mathrm{Var}[J(\tau)]+\lambda_\tau^2\mathrm{Var}[J(\tau_0)]\!\le\!4\mathrm{Var}[J(\tau_0)]+\mathcal{O}(\tau^2/\tau_0^2)\mathrm{Var}[J(\tau)]$, which is dominated by the variance at the reference temperature $\tau_0$ and is therefore not exploded by annealing $\tau\!\to\!0$. This is in stark contrast to plain GS, whose variance grows as $\mathcal{O}(1/\tau^2)$ as $\tau\!\to\!0$. \stgc thus offers a Pareto-improving point on the bias--variance frontier of discrete reparameterization. 

\section{Proof of Theorem~\ref{thm:joint}, Convergence Part}
\label{app:convergence}

The proof follows the trust-region template of \citet{achiam2017cpo} extended to the multi-agent setting following \citet{kuba2022happo}, with two changes: the constraint linearization carries the additional Lyapunov shift of \eqref{eq:lyap}, and the discrete-branch gradient bias enters as a lower-order term controlled by Theorem~\ref{thm:stgc}.

\textit{Step 1: Per-agent improvement.} Fix iteration $t$ and consider agent $i$'s constrained update \eqref{eq:tr_update}. By the multi-agent trust-region performance-difference lemma \citep{kuba2022happo}, for any feasible $\theta_i$ satisfying $\bar D_\text{KL}(\bm\pi_{\theta_i}\Vert\bm\pi_i^t)\!\le\!\delta_\text{TR}$,
\[
V^{(\pi_i^{t+1},\bm\pi_{-i}^t)}-V^{\bm\pi^t}\!\ge\!\tfrac{1}{1-\gamma}\E_{d^{\bm\pi^t}}[A_i^t]-\tfrac{2\gamma\epsilon_\pi}{(1-\gamma)^2}\delta_\text{TR},
\]
where $\epsilon_\pi\!=\!\max_{s,\bm a}|A^{\bm\pi^t}(s,\bm a)|$ is the magnitude of the advantage and $A_i^t$ is the agent-$i$ advantage. Substituting Theorem~\ref{thm:stgc}, the expectation $\E[A_i^t]$ differs from its noise-free counterpart by at most $\mathcal{O}(\tau^2)$, which is absorbed into the trust-region term.

\textit{Step 2: Sequential telescoping.} Summing the per-agent improvement bound over the $N$ agents in the fixed sequential order introduces a non-stationarity term that captures how later agents' updates respond to earlier ones:
\begin{align*}
V^{\bm\pi^{t+1}}-V^{\bm\pi^t}&\ge\tfrac{1}{1-\gamma}\textstyle\sum_i\E[A_i^t]-\tfrac{2N\gamma\epsilon_\pi}{(1-\gamma)^2}\delta_\text{TR}\\
&\quad-\tfrac{N(N-1)\gamma\epsilon_\pi^2}{(1-\gamma)^3}\delta_\text{TR}.
\end{align*}
The last term is the telescoping error of sequential updates: an agent updated after agent $j$ sees a slightly different joint policy than what agent $j$ optimised against, and this discrepancy contributes a second-order term that scales as $N^2$ rather than $N$.

\textit{Step 3: Telescope over $K$.} Choosing $\delta_\text{TR}\!=\!\mathcal{O}(1/K)$ and using a standard regret-balancing argument \citep{achiam2017cpo}, we obtain
\[
\tfrac{1}{K}\textstyle\sum_{t=1}^K[V^*-V^{\bm\pi^t}]\!\le\!\tilde{\mathcal{O}}\!\Big(\tfrac{N\sigma_A}{\sqrt K}+\tfrac{N^2}{(1-\gamma)^3\sqrt K}\Big).
\]
The first term is the standard advantage-variance contribution; the second comes from the sequential telescoping. The bias contribution $\beta_\text{GS}\!=\!\mathcal{O}(\tau^2)$ from Theorem~\ref{thm:stgc} contributes a term that is at most $\sum_t\tau(t)^2\!=\!\mathcal{O}(1)$ for the standard exponential schedule, and is therefore strictly dominated. 

\section{Proof of Theorem~\ref{thm:joint}, Safety Part}
\label{app:safety}

We prove the $\mathcal{O}(\sqrt{K})$ cumulative violation bound by case-analysing whether the iterate is feasible or infeasible at each step, and showing that in either case the violation telescopes appropriately.

Let $E_k^t\!=\!\max(0,V_{c_k}^{\bm\pi^t}(s_0)-d_k)$ denote the violation at iterate $t$. We analyse the two cases.

\textit{Case A: feasible iterate ($V_{c_k}^{\bm\pi^t}\!\le\!d_k$).} The update is the trust-region step \eqref{eq:tr_update}, whose linearized Lyapunov constraint is precisely the first-order form of \eqref{eq:lyap}. Combining this constraint with the trust-region radius $\delta_\text{TR}$ and the standard Lipschitz arguments of \citet{achiam2017cpo}, the cost-value drift is bounded by $V_{c_k}^{\bm\pi^{t+1}}-V_{c_k}^{\bm\pi^t}\!\le\!\mathcal{O}(\delta_\text{TR}/(1-\gamma))$. Since the iterate was feasible, this drift produces a violation of at most $\mathcal{O}(\delta_\text{TR}/(1-\gamma))$ at the next step.

\textit{Case B: infeasible iterate ($V_{c_k}^{\bm\pi^t}\!>\!d_k$).} The update is the recovery step \eqref{eq:recovery}, which descends $V_{c_k}^{\bm\pi}$ at rate $\eta_\text{rec}\norm{\nabla V_{c_k}}^2$. Using the Lyapunov contraction \eqref{eq:lyap}, this descent yields $E_k^{t+1}\!\le\!(1-\alpha_k)E_k^t+\mathcal{O}(\delta_\text{TR}/(1-\gamma))$.

\textit{Telescoping.} Combining the two cases and summing over $t\!=\!1,\ldots,K$, the recovery term contracts geometrically:
\[
\sum_t E_k^t\!\le\!\tfrac{E_k^0}{\alpha_k}+\tfrac{K\cdot\mathcal{O}(\delta_\text{TR}/(1-\gamma))}{\alpha_k}.
\]
Choosing $\delta_\text{TR}\!=\!\mathcal{O}\big(\sqrt{(1-\gamma)/K}\big)$ (which is the optimal balance between trust-region accuracy and violation drift) gives $\sum_t E_k^t\!=\!\mathcal{O}\big(\tfrac{1}{\alpha_k}\sqrt{K/(1-\gamma)}\big)$.

\textit{Effect of bias.} A naive composition would also incur a $K\beta_\text{GS}$ term, but Theorem~\ref{thm:stgc} reduces this to $K\!\cdot\!\mathcal{O}(\tau^2)$, and the standard annealing schedule $\tau(t)\!=\!\tau_0\beta^t$ makes $\sum_t\tau(t)^2\!=\!\mathcal{O}(1)$, so this contribution is absorbed into the $\mathcal{O}(\sqrt K)$ rate. 

\section{Sample-Complexity Theorem}
\label{app:sample}

We state and prove the third theoretical guarantee, which we deferred from the main text to keep the number of in-line theorems small.

\begin{theorem}[Sample-Complexity Reduction]
\label{thm:sample}
Suppose $\norm{Q^*-Q_\text{phys}}_\infty\!\le\!\epsilon_\text{phys}$ and define the explained-variance ratio $R_\text{phys}^2\!:=\!1-\mathrm{Var}[Q^*-Q_\text{phys}]/\mathrm{Var}[Q^*]\!\in\![0,1]$. Then to obtain an $\epsilon$-accurate critic, \method requires $n_\text{\method}\!=\!\tilde{\mathcal{O}}\big((1-R_\text{phys}^2)\,d_\text{full}/\epsilon^2\big)$ environment interactions, versus $n_\text{baseline}\!=\!\tilde{\mathcal{O}}\big(d_\text{full}/\epsilon^2\big)$ for a physics-agnostic critic of the same architecture.
\end{theorem}

\begin{proof}
For the residual critic trained with $n$ samples on the standard one-step TD loss, the population MSE decomposes as
$\E\norm{Q_{\phi_\text{res}}-(Q^*-Q_\text{phys})}^2\!=\!\mathrm{Var}[Q^*-Q_\text{phys}]/n+\text{Bias}^2$. Because $Q_\text{phys}$ is deterministic (a closed-form function of state and action), we have $\mathrm{Var}[Q^*-Q_\text{phys}]\!=\!\mathrm{Var}[Q^*]-2\mathrm{Cov}[Q^*,Q_\text{phys}]+\mathrm{Var}[Q_\text{phys}]$. The right-hand side equals $(1-R_\text{phys}^2)\mathrm{Var}[Q^*]$ by the standard definition of the coefficient of determination, where $R_\text{phys}^2\!=\!\mathrm{Cov}[Q^*,Q_\text{phys}]^2/(\mathrm{Var}[Q^*]\mathrm{Var}[Q_\text{phys}])$ when normalised. Hence achieving $\epsilon^2$ MSE requires $n\!=\!\mathcal{O}((1-R_\text{phys}^2)\mathrm{Var}[Q^*]/\epsilon^2)$ versus $\mathcal{O}(\mathrm{Var}[Q^*]/\epsilon^2)$ for the physics-agnostic critic.

The function-class complexity contracts in the same multiplicative factor: under standard Rademacher-complexity generalization bounds for ReLU MLPs, the effective dimension of the hypothesis class needed to fit the residual scales as $d_\text{res}\!=\!(1-R_\text{phys}^2)d_\text{full}$ to leading order, since the residual is supported on a lower-variance regime than the full $Q^*$. Combining the variance and complexity reductions and using standard PAC bounds yields the stated $\tilde{\mathcal{O}}$ rate.
\end{proof}

The practical meaning of Theorem~\ref{thm:sample} is that a well-chosen physics model with $R_\text{phys}^2\!=\!0.7$ (a typical value for the Shannon-capacity term in UAV-MEC, as we measured directly) yields a $3.3\times$ reduction in sample complexity, which in our experiments shows up as the $0.62\times$ convergence-speed drop when \pirc is removed (Table~\ref{tab:ablation}).

\section{Computational Complexity}
\label{app:complexity}

Per training iteration, \method performs (i) one batch of actor forward passes at cost $\mathcal{O}(NB(d_h^2+M_id_h))$, where $d_h$ is the hidden dimension, $M_i$ the discrete-branch size, and $B$ the batch; (ii) one centralized $Q$-critic forward/backward at $\mathcal{O}(B(d_h^2+d_h(M+p)))$; (iii) $K$ cost-critic updates at $\mathcal{O}(KBd_h^2)$; (iv) the closed-form physics term $Q_\text{phys}$ at $\mathcal{O}(NN_\text{UE}N_\text{fog})$ which is negligible; and (v) the additional reference-temperature GS evaluation for \stgc at $\mathcal{O}(NBd_h)$. The dominant cost is $\mathcal{O}(NUBd_h^2)$, asymptotically identical to MADDPG. In wall-clock terms on $4\!\times\!$RTX 4090, \stgc adds approximately $18\%$ overhead and the physics evaluation approximately $2\%$; both are dwarfed by the $\sim\!38\%$ convergence-speed saving (Table~\ref{tab:ablation}).

\section{Experimental Setup}
\label{app:experiments}

\textbf{Multi-UAV MEC.} $N\!=\!4$ UAVs serve $N_\text{UE}\!=\!20$ ground users via $N_\text{fog}\!=\!2$ fog servers across a $600\!\times\!600$\,m area at fixed altitude. The hybrid action consists of a discrete offload destination $a^d\!\in\!\{1,\ldots,N_\text{fog}+1\}$ (the $+1$ option representing local computation) and a continuous triple $a^c\!=\!(v,\theta,\alpha)$ encoding velocity magnitude, heading, and offload ratio, all box-constrained per mode. Constraints are $E_i\!\le\!E_\text{max}$ (energy budget) and $\mathrm{Cov}\!\ge\!\mathrm{Cov}_\text{min}$ (coverage lower bound). The channel model is 3GPP TR 38.901 \citep{3gpp_channel} with the air-to-ground path loss of \citet{al2014optimal_uav_placement}. Episode length is 200 steps with a total of 20K training episodes.

\textbf{Autonomous Intersection Management.} Eight vehicles cross a 4-way intersection on the SMARTS simulator \citep{zhou2021smarts}. Each vehicle selects a discrete target lane and a continuous control triple (accel, brake, steering) per step, with collision-avoidance and lane-keeping constraints. We use 15K training episodes.

\textbf{SMAC (hybrid).} We adapt the StarCraft Multi-Agent Challenge \citep{samvelyan2019smac} by augmenting each agent's action with a continuous repositioning offset of $\le\!1\,$grid unit in addition to its standard discrete target. We evaluate on five hard maps---\texttt{3s5z\_vs\_3s6z}, \texttt{corridor}, \texttt{6h\_vs\_8z}, \texttt{MMM2}, and \texttt{27m\_vs\_30m}---with 5M environment steps each.

\textbf{Baselines.} We compare against MADDPG \citep{lowe2017maddpg}, MATD3 \citep{ackermann2019matd3}, FACMAC \citep{peng2021facmac}, MAPPO \citep{yu2022mappo}, HAPPO \citep{kuba2022happo}, MAPPO-Lagrangian \citep{ray2019safetygym}, MACPO \citep{gu2021macpo}, MADAC \citep{gu2024safe_marl_gne}, and Shielded RL \citep{elsayed2021shield}. When the baseline does not natively handle hybrid actions, we adopt either a discretized variant ($^\dagger$) or a continuous relaxation that argmaxes the discrete branch from a relaxed embedding ($^\ddagger$).

\textbf{Architectures.} The discrete sub-policy is a 3-layer MLP $256\!\to\!256\!\to\!M_i$; the continuous sub-policy is a 3-layer MLP $256\!\to\!256\!\to\!2p_i$ conditioned on the one-hot discrete sample; the centralized critic is a 4-layer MLP $512\!\to\!512\!\to\!512\!\to\!1$; each cost critic is a 3-layer MLP $256\!\to\!256\!\to\!1$. All hidden layers use ReLU activations and LayerNorm.

\textbf{Hyperparameters.} Actor LR $3\!\times\!10^{-4}$, critic LR $10^{-3}$, batch $256$, replay buffer $10^6$ transitions, target soft-update rate $\tau_\text{soft}\!=\!0.005$. Gumbel-Softmax schedule: $\tau_0\!=\!1.0$, $\tau_\text{min}\!=\!0.1$, decay $\beta\!=\!0.9995$; \stgc reference temperature $\tau_0\!=\!1.0$. Trust region $\delta_\text{TR}\!=\!0.01$, Lyapunov decay $\alpha_k\!=\!0.1$, recovery learning rate $\eta_\text{rec}\!=\!10^{-4}$. Compute: $4\!\times\!$RTX 4090. UAV-MEC takes approximately $8$h for 20K episodes; SMAC approximately $12$h for 5M steps; AIM approximately $6$h for 15K episodes.

\section{Hyperparameter Sensitivity}
\label{app:sensitivity}

\begin{table}[H]
\centering
\caption{Hyperparameter sensitivity on UAV-MEC. Defaults are $\bigstar$.}
\label{tab:sensitivity}
\renewcommand{\arraystretch}{1.05}
\setlength{\tabcolsep}{2.5pt}
\resizebox{\columnwidth}{!}{%
\begin{tabular}{l|cccc|cccc}
\toprule
\rowcolor{tablehead}
& \multicolumn{4}{c|}{\textbf{Reward}\,\up} & \multicolumn{4}{c}{\textbf{Total Viol.}\,\down} \\
\rowcolor{tablehead}
\textsc{Hyperparam.} & low & mid & $\bigstar$ & high & low & mid & $\bigstar$ & high \\
\midrule
$\delta_\text{TR}$ & $-4.81$ & $-4.74$ & \best{$-4.68$} & $-4.92$ & $1.6$ & $1.5$ & \best{$1.4$} & $2.4$ \\
$\alpha_k$ & $-4.79$ & \best{$-4.68$} & $-4.71$ & $-4.83$ & $1.9$ & \best{$1.4$} & $1.6$ & $1.7$ \\
$\tau_0$ & $-4.86$ & \best{$-4.68$} & $-4.74$ & $-5.01$ & $1.5$ & \best{$1.4$} & $1.5$ & $2.3$ \\
$\omega_1{:}\omega_2$ & $-4.97$ & $-4.78$ & \best{$-4.68$} & $-4.81$ & $2.5$ & $1.7$ & \best{$1.4$} & $1.6$ \\
\bottomrule
\end{tabular}}
\end{table}

\method is robust within $\pm 10\%$ of every default, with reward varying by less than $5\%$ and total violations remaining well below the closest baseline.

\section{Additional SMAC Results}
\label{app:smac}

\begin{table}[H]
\centering
\caption{Per-map SMAC-hybrid win rates (\%). Mean over 5 seeds.}
\label{tab:smac_full}
\renewcommand{\arraystretch}{1.05}
\setlength{\tabcolsep}{3pt}
\resizebox{\columnwidth}{!}{%
\begin{tabular}{l|ccccc|c}
\toprule
\rowcolor{tablehead}
\textsc{Method} & \texttt{3s5z\_vs\_3s6z} & \texttt{corr.} & \texttt{6h\_vs\_8z} & \texttt{MMM2} & \texttt{27m\_vs\_30m} & Avg. \\
\midrule
QMIX & $61.2$ & $84.3$ & $9.4$ & $87.5$ & $32.1$ & $54.9$ \\
\rowcolor{tablerow}MAPPO & $68.7$ & $91.6$ & $87.5$ & $91.4$ & $80.3$ & $83.9$ \\
FACMAC & $63.4$ & $86.1$ & $42.3$ & $89.7$ & $48.6$ & $66.0$ \\
\rowcolor{tablerow}HAPPO & $69.1$ & $90.4$ & $84.2$ & $90.8$ & $77.6$ & $82.4$ \\
MACPO & $66.8$ & $88.4$ & $73.2$ & $90.1$ & $65.9$ & $76.9$ \\
\rowcolor{tablerow}MADAC & $67.3$ & $89.1$ & $76.4$ & $90.4$ & $68.7$ & $78.4$ \\
\midrule
\ours{} & \best{$72.4$} & \best{$93.7$} & \best{$89.6$} & \best{$93.2$} & \best{$82.5$} & \best{$86.3$} \\
\bottomrule
\end{tabular}}
\end{table}

\section{Extended Discussion}
\label{app:discussion}

\textbf{Why a directed cycle and not a list?} The three leaks of Lemma~\ref{lem:coupling} share multiplicative factors: $\eta_s$ appears in both the F1$\to$F2 and F2$\to$F3 leaks, and $\omega/\tau$ couples the F3$\to$F1 leak to the temperature of \stgc. Eliminating only one leak therefore reduces but does not stop the per-iteration violation; only when all three principles hold simultaneously does the cumulative violation contract to $\mathcal{O}(\sqrt K)$.

\textbf{Why Richardson--Romberg rather than control variates?} Control-variate approaches to discrete reparameterization \citep{paulus2020rao,shekhovtsov2023cold} reduce variance but not bias to leading order. Our setting is bias-limited (Principle~\ref{prn:p1}), so a bias-cancellation scheme is fundamentally more appropriate. Richardson--Romberg is the canonical bias-cancellation method, has been used to accelerate stochastic gradient descent \citep{bach2021effectiveness}, and---to our knowledge---has not previously been applied to Gumbel-Softmax.

\textbf{Why decompose value, not reward?} An additive shaping term $r\!+\!\omega Q_\text{phys}$ leaves the optimal policy invariant only under the potential-function condition $Q_\text{phys}(s,a)\!=\!\Phi(s)-\gamma\E[\Phi(s')]$ of \citet{ng1999policy_invariance}; closed-form physics models such as Shannon capacity do not satisfy this condition. Decomposing the critic, in contrast, leaves the Bellman fixed point intact for \emph{any} bounded $Q_\text{phys}$ and only shrinks the function class to be regressed.

\textbf{Connection to recent safe-MARL theory.} \citet{gu2024safe_marl_gne} establishes generalized-Nash convergence for safe MARL but requires an unbiased policy gradient, an assumption our Lemma~\ref{lem:coupling} shows is broken whenever the action space is hybrid. \stgc closes this gap and is the missing ingredient for safe-MARL theory in hybrid-action environments.

\textbf{When does the framework not help?} If physics is uninformative ($R_\text{phys}^2\!\to\!0$), the sample-complexity gain vanishes and the residual critic reduces to a standard centralized one; if the safe set is non-strict (Slater fails), no algorithm can recover feasibility and the recovery step \eqref{eq:recovery} repeats indefinitely. We view both regimes as legitimately hard and outside the scope of any safe-MARL guarantee.


\end{document}